\documentclass[11pt,journal,compsoc,onecolumn]{IEEEtran}
 \usepackage{cite}

\usepackage{amsmath,graphicx}
\usepackage{color}
 \usepackage{amssymb}
\usepackage{amsfonts}
\usepackage{dsfont}
\usepackage{epsfig}
\newtheorem{definition}{Definition}
\newtheorem{proposition}{Proposition}
\newtheorem{proof}{Proof}
\usepackage{gensymb}
\usepackage{multirow}

\ifCLASSINFOpdf
\else
\fi
\ifCLASSOPTIONcompsoc
 \usepackage[font=normalsize,labelfont=sf,textfont=sf]{subfig}
\else
 \usepackage[font=footnotesize]{subfig}
\fi
\ifCLASSOPTIONcompsoc
 \usepackage[font=normalsize,labelfont=sf,textfont=sf]{subfig}
\else
 \usepackage[font=footnotesize]{subfig}
\fi

\usepackage{color}
 \usepackage{amssymb}
\usepackage{amsfonts}
\usepackage{dsfont}
\usepackage{epsfig}

\title{Human Action Recognition with Multi-Laplacian Graph Convolutional Networks}
\author{Ahmed Mazari \ \ \ \ \ \ \ \ \ \ \ \ \ \ Hichem Sahbi \\ $ $ \\ {Sorbonne University, UPMC, CNRS, LIP6, F-75005 Paris, France}}

\begin{document}

\maketitle

\begin{abstract}

Convolutional neural networks are nowadays witnessing a major success in different pattern recognition problems. These learning models were basically designed to handle vectorial data such as images but their extension to non-vectorial and semi-structured data (namely graphs with variable sizes, topology, etc.) remains a major challenge, though a few interesting solutions are currently emerging. \\ 
In this paper, we introduce MLGCN; a novel spectral Multi-Laplacian Graph Convolutional Network. The main contribution of this method resides in a new design principle that learns graph-laplacians as convex combinations of other elementary laplacians -- each one dedicated to a particular topology of the input graphs. We also introduce a novel pooling operator, on graphs, that proceeds in two steps: context-dependent node expansion is achieved, followed by a global average pooling; the strength of this two-step process resides in its ability to preserve the discrimination power of nodes while achieving permutation invariance. Experiments conducted on SBU and UCF-101 datasets, show the validity of our method for the challenging task of action recognition. \\

\end{abstract}
\def\S{{\cal S}} 
\def\V{{\cal V}}
\def\G{{\cal G}}
\def\E{{\cal E}}
\def\L{{\bf L}}
\def\I{{\bf I}}
\def\D{{\bf D}}
\def\A{{\bf A}}
\def\U{{\bf U}}
\def\Lambdaa{{\Lambda}}
\def\T{{\bf T}}
\def\N{{\cal N}}

\def\thetaa{{\bf \theta}}
\def\betaa{{\bf \hat{w}}}

%-------------------------------------------------------------------------
\section{Introduction}
\label{sec:intro}
Video action recognition is a major task in computer vision which consists in classifying sequences of frames into categories (or classes) of actions. This task is known to be challenging due to the intrinsic properties (appearance and motion) of moving objects and also their extrinsic acquisition conditions (occlusions, background clutter, camera motion, illumination, length/resolution, etc.). Most of the existing  action recognition methods are based on machine learning~\cite{lingsahbi2013,temporalpyramid,lingsahbieccv2014,mkl_action,temporal_pyramid,superived_dic_action,multi_svm,lingsahbiicip2014}; their general recipe consists  in extracting (handcrafted or learned) features  \cite{sahbijstars17,sahbispie2004,sahbiicip09,sahbiigarss16,sahbiigarss11} prior to classifying them  using inference techniques \cite{sahbiicassp15,sahbiicassp16a} such as kernel methods and deep networks \cite{sahbiaccv2010,twostream14,kin3d,sahbipr2012,pose,sahbikpca06,spresnet16,mklimage2017,spresnetmulti17,sahbiicassp16b,sahbiicassp16b,sahbiclef13,tristan2017,TSN,sahbiclef08,sahbiphd, sahbiiccv17,sahbiicip18}. \\ 
\indent Among the machine  learning techniques -- for action recognition -- those based on deep networks are particularly performant;  successful methods include two-stream 2D convolutional neural networks (CNNs) \cite{twostream14}, two-stream 3D CNNs and simple 3D CNNs \cite{kin3d}. However, and beside being data-hungry, these models rely on a strong assumption that videos are described as vectorial data; in other words, these methods assume that videos come only in the form of regular (2D or 3D) grids. This  assumption may not  hold in practice: on the one hand,  one may consider moving objects as constellations  of interacting body parts (such as 2D/3D skeletons or joints in human actions) and this requires processing only these joints without taking into account {\it holistically}  cluttered background or other parts in the scenes. On the other hand, moving objects may be occluded with spurious details which are not necessarily related to the moving object parts. Hence, for these particular settings, graph convolutional networks (GCNs) \cite{Node_sampling} are rather more appropriate where nodes, in these models, capture object parts and links their spatio-temporal interactions. \\
\noindent Early GCNs are targeted to graphs with  known/fixed topology\footnote{as 2D regular grids (see also \cite{sahbispie05,sahbicassp11,sahbiicpr18}).} (fixed number of nodes/edges, constant degree, etc.)~\cite{Bresson16,Henaff15}; in existing solutions pixels are considered as nodes and edges connect neighboring pixels.  Despite their relative success for some pattern classification tasks including optical character recognition (on widely used benchmarks such as MNIST), these methods do not straightforwardly extend to general graphs with arbitrary topological characteristics (variable number of nodes/edges, heterogeneous degrees, etc.) and this limits their applicability to other challenging tasks such as action recognition.  Recent attempts, to extend these methods to action recognition  \cite{STGCN,ASGCN,PartGraph}, include \cite{STGCN}  which models connectivity of moving joints in videos using graphs where nodes correspond to joints (described by spatial coordinates and their likelihoods) and edges characterize their spatio-temporal interactions. One of the drawbacks of these extensions resides in the limited representational power of joints and also the difficulty in achieving permutation invariance; in other words, parsing and describing joints while being invariant to arbitrary reordering of objects especially for highly complex scenes with multiple interacting objects/persons.  From the machine learning point of view, GCN operates either directly in the spatial domain \cite{Petar18,Monti17,Bronstein16,SAGECONV18,AGNNCONV18,SGCCONV19,APPNP19,ECC,monte_carlo,covariant_GNN,motion_graph,GRAPHCONV18,GAT18,SplineCNN} or require a preliminary step of spectral decomposition of graphs using Fourier basis \cite{Ortega13,Ortega18,irregularity_GSP,vertex_analysis}  prior to achieve convolution \cite{Bresson16,kipf17,Henaff15,bruna13,Bronstein17,ARMACONV19,wavelet_GCN,scattering_GCN,diff_scattering}. While graph convolution in the  spectral domain is well defined, its success heavily relies on the choice of the laplacian operators \cite{Laplacian-topology}  that capture the topology of the manifolds enclosing data.  These laplacians, in turn, depend on many hyper-parameters which are difficult to set using tedious cross-validation especially when training GCNs on large-scale datasets. \\
\indent In this paper, we address the aforementioned issues (mainly laplacian design in GCNs and permutation invariance) for the particular task of   action recognition. Our solution achieves  convolution in the spectral domain using a new design principle that considers a convex combination of  several laplacian operators; each laplacian is dedicated to a particular (possible) topology of our graphs. We also introduce a novel context-dependent pooling operator  that proceeds in two steps: node features are first expanded with their contexts and then globally averaged; the strength of this two-step pooling process resides in its ability to preserve/enhance the discrimination power of node representations while achieving permutation invariance. The validity of these contributions is corroborated through extensive experiments, in action recognition, using the challenging SBU-skeleton and UCF-101 datasets.  
\section{Graph Construction}\label{graphc}
In this section, we briefly  describe the video processing used to build our input graphs. This step consists in extracting and grouping joints (a.k.a keypoints) into trajectories prior to modeling their spatio-temporal interactions with graphs. \\
\indent Given a raw video, skeletons are obtained by detecting human joints in successive frames using the state of the art human pose extractor~\cite{Yaser17}\footnote{This processing is only reserved to raw video datasets (including UCF~\cite{UCF101}) while for other databases, such as SBU~\cite{SBU12},  skeletons are already available.};  as these keypoints are labeled (see Fig.~\ref{fig1}), their trajectories are extracted by simply tracking keypoints with the same labels. Considering a finite collection of trajectories, we build an adjacency graph $\G=(\V,\E)$ where each node $v \in \V$ corresponds to a labeled trajectory and an edge $(v,v') \in \E$ exists between two nodes iff the underlying trajectories are spatially neighbors. Each trajectory (i.e., node in $\G$) is described by aggregating motion and appearance streams as shown subsequently.\\
\begin{figure*}[hpbt]
  \begin{center}
    \centerline{\scalebox{0.49}{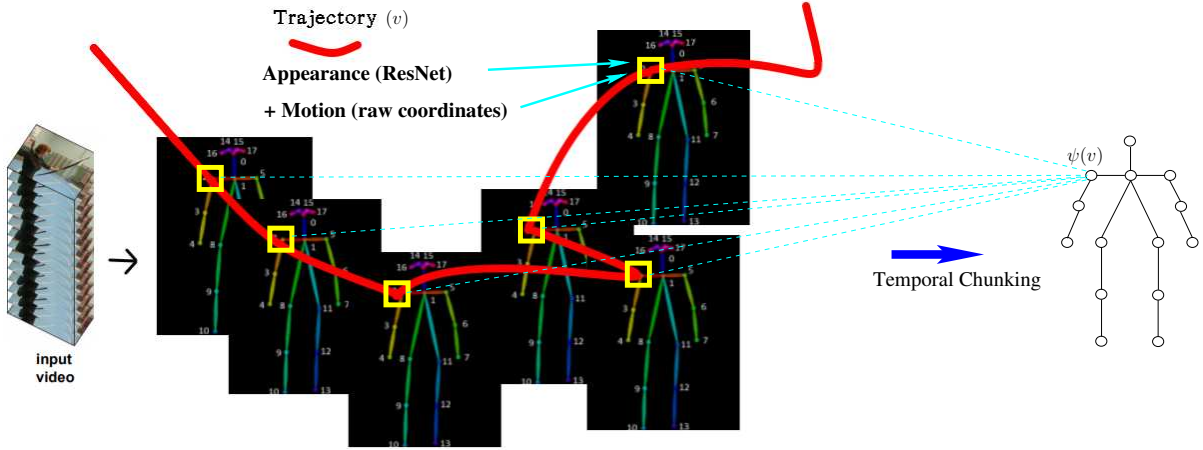}}
\caption{This figure shows the whole keypoint extraction, tracking and description process on motion and appearance streams {(see also the detailed protocol in the supplementary material \cite{SMBM19}).}} \label{fig1}
\end{center}
\end{figure*}

\noindent {\bf Motion stream.} Considering a video as a sequence of skeletons, we process the underlying trajectories using {\it temporal chunking}: first we split the total duration of a  video into $C$ equally-sized temporal chunks ($C=4$ in practice), then we assign  the keypoint coordinates of  a given trajectory $v$  to the $C$ chunks (depending on their time stamps) prior to concatenate the averages of these chunks and this produces the description of $v$ denoted as $\psi(v)$. Hence, trajectories,  with similar keypoint coordinates but arranged differently in time, will be considered as very different. Note that beside  being compact and discriminant (as shown later in table 3(c)), this temporal chunking gathers advantages --  while discarding drawbacks -- of two widely used families of techniques mainly {\it global averaging techniques} (invariant but less discriminant)  and  {\it frame resampling techniques} (discriminant but less invariant). Put differently, temporal chunking produces discriminant descriptions that preserve the temporal structure of trajectories while being {\it frame-rate} and {\it duration} agnostic. \\
\noindent {\bf Appearance stream.} Similarly to motion, we also describe each trajectory using appearance features. First, we apply ResNet~\cite{resnet_imagenet} frame-wise\footnote{We consider a local neighborhood around each keypoint in order to extract these convolutional features.} in order to collect convolutional features associated to different keypoints (see again Fig.~\ref{fig1}), then we aggregate those convolutional features through trajectories using temporal chunking as described above for motion stream.

\section{Multi-Laplacian Convolutional Networks}
\indent Given a collection of videos, we describe each one using a graph $\G_i=(\V_i,\E_i)$ as shown in section~\ref{graphc}. For each node $v \in \V_i$, we extract two feature vectors, denoted $\psi_m(v)$, $\psi_a(v)$, respectively corresponding to  motion and appearance streams of $v$. We also define a similarity between nodes in $\V_i$ as $k_m(v,v')=\exp(-\|\psi_m(v)-\psi_m(v')\|_2^2/\sigma_m)$, here $\sigma_m$ is the scale of the gaussian similarity and $\|.\|_2$ is the $\ell_2$ norm. Similarly, we define   $k_a(v,v')$ using appearance features. In the remainder of this paper, unless explicitly mentioned, we denote a given graph $\G_i$ simply as $\G$. We also denote  motion and appearance features $\psi_m(v)$, $\psi_a(v)$ as $\psi(v)$, scales $\sigma_m$, $\sigma_a$  as $\sigma$, and similarities $k_m(v,v')$, $k_a(v,v')$ as $k(v,v')$.\\
\indent The goal is to design a GCN that returns the representation and the classification of a given graph. This includes  {\it a novel design of laplacian convolution and pooling} on graphs as shown subsequently.
\subsection{Spectral  graph convolution at a glance} \label{conv0}
\noindent  Given a graph $\G=(\V,\E)$ with $|\V|=n$, $|\E|$ being respectively the number of its vertices and edges and  $\L$ the laplacian of $\G$; for instance,  $\L$ could be the normalized, unormalized or random walk laplacians respectively defined as  $\L=\I_{n}-\D^{-1/2}\ \A \ \D^{-1/2}$,   $\L=\D- \A$ and $\L=\D^{-1} \A$ where $\I_{n}$ is an $n \times n$ identity matrix, $\A$ is the affinity matrix built as $[\A]_{vv'}=\mathds{1}_{\{(v,v')\in \E\}}$ or by using the gaussian similarity $k(.,.)$ as $[\A]_{vv'}=\mathds{1}_{\{(v,v')\in \E\}}.k(\psi(v),\psi(v'))$ and  $\D$ a diagonal degree matrix with each diagonal entry $[\D]_{vv}= \sum_{v'} [\A]_{vv'}$. Considering the eigen-decomposition of $\L$ as $\U \Lambdaa \U'$ with $\U$, $\Lambdaa$ being respectively the matrix of its eigenvectors (graph Fourier modes) and the diagonal matrix of its non-negative eigenvalues, spectral graph convolution is a well defined operator  (see for instance \cite{Bresson16})  which is achieved by first projecting a given graph signal $\psi(.)$ using the eigen-decomposition of  $\L$, and then multiplying the resulting projection by a convolutional filter prior to back-project the result in the original signal space. \\
\indent Formally, the convolutional operator $\star_{\G}$ (rewritten for short as $\star$) on the graph signal $\psi(\V) \in \mathbb{R}^{n \times p}$  is  $ (\psi \star  g_{\thetaa})(\V) =  \U \ g_{\theta}(\Lambdaa) \ \U' \ \psi(\V)$;  here  $g_{\theta}$ denotes a non-parametric convolutional filter defined as $g_{\thetaa}(\Lambdaa)={diag}(\thetaa)$ with $\thetaa \in \mathbb{R}^{n}$. As this filter is not localized, we consider instead \cite{Bresson16} 
        \begin{equation}\label{eq00}
           (\psi \star  g_{\thetaa})(\V) := \sum_{k=0}^{K-1} \theta_{k} \ T_{k}(\L)  \ \psi(\V),  
        \end{equation} 
        \noindent with $K$ fixed and $\thetaa=(\thetaa_1 \dots \thetaa_K)' \in  \mathbb{R}^{K}$ being its learned convolutional filter parameters; in practice, we consider a rescaled version of the laplacian (i.e., $2 \L / \lambda_{max} - \I_{n}$  instead of $\L$ with  $\lambda_{max}$  being its largest eigenvalue). In the above equation,  $T_{k}$ is the $k$-th order Chebyshev polynomial recursively defined as $T_{k}(\L)=2 \L \ T_{k-1}(\L)-T_{k-2}(\L)$, with $T_{k}(\L) \in  \mathbb{R}^{n\times n}$  and $T_{0}=\I$, $T_{1}=\L$ (for more details see again \cite{Bresson16}).
\subsection{Multi-Laplacian design}

The success of the aforementioned  convolutional process is highly dependent on the {\it relevance} of the used laplacian, which in turn depends on the appropriate choice of the affinity matrix of the graph and its hyper-parameters. Hence, knowing a priori which parameter to choose could be challenging and usually relies on the tedious cross-validation. \\
\indent Our alternative contribution in this paper aims at designing convolutional laplacian operators while learning the topological structure of the input graphs (characterized by their laplacians). Starting from different {\it elementary} laplacians\footnote{also referred to as single or individual laplacians.} associated to multiple settings (for instance, by varying the scale $\sigma$ of the gaussian similarity $k(.,.)$ and the laplacians), we train  a {\it multiple laplacian} as a deep nonlinear combination of multiple elementary laplacians. Fig.~\ref{fig2} shows our learning framework with $d$-layers in the multi-laplacian; for each layer $\ell+1$ ($\ell \in \{0,\dots,d-1\}$) and its associated unit $p \in \{1,\dots, n_{\ell+1}\}$, a laplacian (denoted $\L_{p}^{\ell+1}$) is recursively defined as
\def\w{{\bf w}}
\begin{equation}\label{eq00}
  \L_p^{\ell+1} = g\bigg(\sum_{q=1}^{n_\ell} \ \w_{q,p}^{\ell} \  \L_q^{\ell}\bigg),
\end{equation} 
\noindent where $g$ is a nonlinear activation function (see details in section~\ref{definite}), $n_{\ell}$ is the number of units in layer $\ell$ and $\{\mathbf{w}^{\ell}_{q,p}\}_q$ are the (learned) weights associated to ${\L}_{p}^{\ell+1}$. For any given graph $\G$, a tensor of multiple elementary laplacians  $\{{\L}_{q}^{1}\}_q$ (associated to different combinations of $\{\sigma\}$ and standard laplacians namely unormalized, normalized, random walk, etc.) on $\G$  is considered as an input to our deep network. These elementary laplacians are then forwarded to the subsequent intermediate layer resulting into $n_2$ multiple laplacians through the nonlinear combination of the previous layer, etc. The final laplacian ${\L}_{1}^{d}$  is a highly nonlinear combination of elementary laplacians. We notice that the deep laplacian network in essence is a multi-layer perceptron (MLP), with nonlinear activation functions which is fed (together with the graph signal $\psi(\V)$) as input in order to achieve convolution (see Fig.~\ref{fig2}). Hence, we can use standard  backpropagation in order to optimize the parameters of both the MLP and the GCN networks. Let $J$ denotes the loss function associated to our classification problem (namely cross-entropy); starting from the gradients of this loss $J$ w.r.t the final softmax output, we use the chain rule in order to backpropagate the gradients w.r.t different layers and parameters (fully connected and convolutional layers as well as the MLP of the multi-laplacians), and to update these parameters accordingly using gradient descent. 
\begin{figure*}[hpbt]
\centerline{\scalebox{0.55}{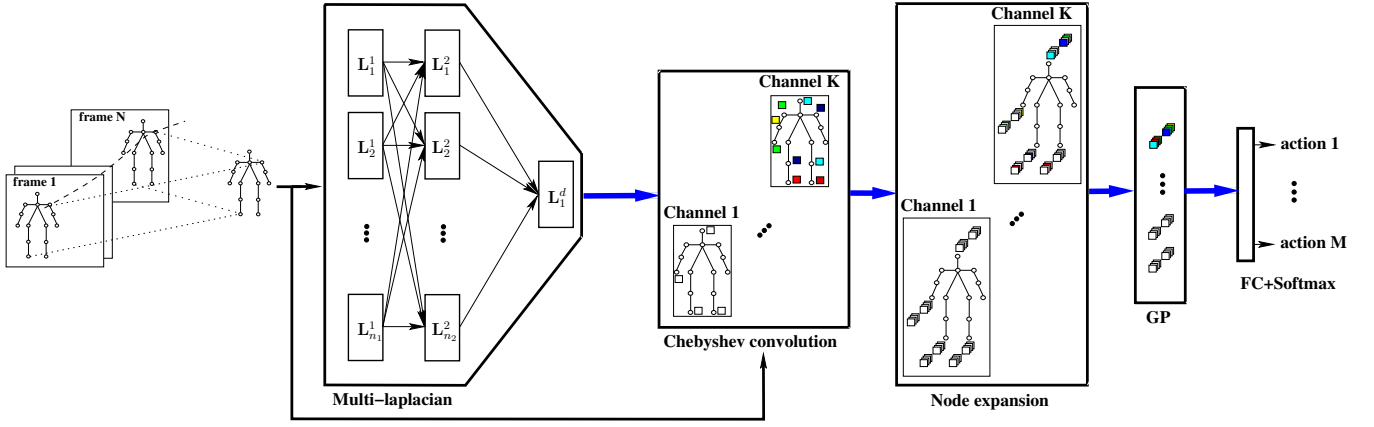}}
\caption{This figure shows the architecture of our multi-laplacian graph convolutional network (MLGCN). First, multiple elementary laplacians (associated to  $\G=(\V,\E)$) and graph signal $\psi(\V)$ are fed as input to an MLP in order to learn the best combination of laplacians. Then, Chebyshev decomposition  is achieved using the learned multi-laplacian in order to perform graph convolution, followed by node expansion and global average pooling prior  to softmax classification {\bf (better to zoom the pdf)}.} \label{fig2}
\end{figure*}

\subsection{Activation functions and optimization}\label{definite}
We consider two activation functions $g$ in Eq.~(\ref{eq00}): ReLU and leaky ReLU~\cite{softplus1,softplus2,leaky}. Note that only leaky ReLU provides negative entries in the learned laplacians and both of these activations allow  learning {\it conditionally} positive definite (c.p.d) laplacian matrices. In what follows, we discuss  the sufficient conditions about the choices of the elementary input laplacians, the parameters  $\{{\bf w}_{q,p}^\ell\}$ and the activation functions that guarantee this c.p.d property.
\begin{definition}[conditionally positive definite laplacians]\label{def0} 
A laplacian matrix $\L$ is conditionally positive definite, iff $\forall c_1,\dots, c_n \in \mathbb{R}$ {(\bf with $\sum_{i=1}^n c_i = 0$)}, $\sum_{i,j} c_i c_j \L_{ij} \geq 0$.
\end{definition} 
From the above definition, it is clear that any positive definite laplacian is also c.p.d. The converse is not true, however c.p.d is a weaker (but sufficient) condition  in order to derive positive definite laplacians (see following propositions).
\def\J{{J}}
\begin{proposition}[Berg et al.\cite{berg84}]\label{prop1}
  Consider  $\L_{i,j}$ as an entry of a matrix $\L$ and define $\hat{\L}$ with
\begin{eqnarray}\label{eq0101}
  \hat{\L}_{i,j}&=& \L_{i,j} - \L_{i,n+1} -  \L_{n+1,j} +  \L_{n+1,n+1}
\end{eqnarray}
Then, $\hat{\L}$ is positive definite if and only if  ${\L}$ is c.p.d.
\end{proposition} 
\begin{proof} See the supplementary material \cite{SMBM19}. Now we derive our main result: \end{proof} 
\begin{proposition}\label{prop0}
Provided that the input elementary laplacians $\{\L_q^{1}\}_{q}$ are c.p.d, and  $\{{\bf w}_{q,p}^{\ell}\}_{p,q,\ell}$ belong to the positive orthant of the parameter space, any combination $g(\sum_q {\bf w}_{q,p}^{\ell} \ \L_q^{\ell})$, with $g$ equal to ReLU or leaky ReLU, is also c.p.d.
\end{proposition} 
\begin{proof} 
\noindent Details of the first part of the proof, based on recursion, are omitted and result from the application of definition~(\ref{def0}) to $\L=\sum_q {\bf w}_{q,p}^{\ell} \ \L_q^{\ell}$ (for different values of $\ell$) while considering $\{\L_q^{1}\}_q$ c.p.d. Now we show the second part of the proof (i.e., if $\L$ is c.p.d, then $g(\L)$ is also c.p.d for ReLU and leaky ReLU).\\

\noindent i) For $g(\L)=\log(1+\exp(\L))$ [ReLU]: considering $\L$ c.p.d, and following proposition~(\ref{prop1}), one may define a positive definite $\hat{\L}$ and obtain $\forall \{c_i\}$
\begin{eqnarray*}
 \sum_{i,j=1}^n c_i c_j \exp(\L_{i,j})  & = &  \exp(\L_{n+1,n+1})  \sum_{i,j=1}^n (c_i \exp(\L_{i,n+1})) .  (c_j \exp(\L_{n+1,j})) . \exp(\hat{\L}_{i,j} \big)\\
  &\geq &0 
 \end{eqnarray*}
so $\exp(\L)$ is also positive definite. Besides, for any arbitrary $\alpha>0$,  $(1+\exp(\L))^{\circ \alpha}$ is also  positive definite with $\circ \alpha$ being the entrywise matrix power. By simply rewriting $(1+\exp(\L))^{\circ \alpha}=\exp(\alpha \ g(\L))$, it follows (from \cite{Shoenberg38}) that $g(\L)$ is c.p.d since $\exp(\alpha \ g(\L))$ is positive definite for all $\alpha>0$.\\

\noindent ii) For $g(\L)=\log(\exp(a \L) +\exp(\L))$ with $0<a\ll 1$  [leaky-ReLU]: one may write $g$ as
\begin{eqnarray}\label{eq1112}
g(\L)=a \ \L + \log(1 +\exp((1-a) \ \L)).
\end{eqnarray}
Since $\exp(\L)$ is positive definite, it follows that $(1+\exp((1-a) \ \L))^{\circ \alpha}$ is also positive definite for any arbitray $\alpha >0$ and $0<a\ll1$  so from \cite{Shoenberg38}, $\log(1 +\exp((1-a) \ \L))$ is c.p.d and so is $g(\L)$; the latter results from the closure of the c.p.d with respect to the sum. \\ 
\begin{flushright}$\blacksquare$ \end{flushright}
\end{proof}

From proposition~(\ref{prop0}), provided that i) the elementary laplacians are c.p.d, ii) the activation function $g$ preserves the c.p.d (as  ReLU and leaky-ReLU) and iii) weights $\{{\bf w}_{q,p}^{\ell}\}$  are positive, all the resulting multiple laplacians in Eq.~\ref{eq00} will also be c.p.d and admit equivalent positive definite laplacians (following proposition~\ref{prop1}), and thereby spectral graph convolution can be achieved. Note that conditions (i) and (ii) are satisfied by construction while condition (iii) requires adding equality and inequality constraints to Eq.~\ref{eq00}, i.e., $\w_{q,p}^{\ell} \in [0,1]$ and $\sum_{q}  \w_{q,p}^{\ell}=1$.  In order to implement these constraints, we consider a reparametrization in Eq.~\ref{eq00} as $\w_{q,p}^{\ell}= f(\betaa_{q,p}^{\ell})\slash {\sum_q f(\betaa_{q,p}^{\ell})}$ for some $\{\hat{\w}_{q,p}^{\ell}\}$
with $f$ being strictly monotonic real-valued (positive) function and this allows free settings of the parameters  $\{\hat{\w}_{q,p}^{\ell}\}$ during optimization while guaranteeing $\w_{q,p}^{\ell} \in [0,1]$ and $\sum_{q}  \w_{q,p}^{\ell}=1$. During backpropagation, the gradient of the loss $J$ (now w.r.t $\betaa$'s) is updated using the chain rule as
\begin{equation}
  \begin{array}{lll}
    \frac{\partial J}{\partial \betaa_{q,p}^{\ell}} &=& \frac{\partial J}{\partial \w_{q,p}^{\ell}} . \frac{\partial \w_{q,p}^{\ell}}{\partial \betaa_{q,p}^{\ell}}  \ \ \ \ \ \ \textrm{with}  \ \ \  \frac{\partial \w_{q,p}^{\ell}}{\partial \betaa_{q,p}^{\ell}} =\frac{f'(\betaa_{q,p}^{\ell})f(\sum_{r\neq q } \betaa_{r,p}^{\ell})}{(f( \betaa_{q,p}^{\ell})+f(\sum_{r\neq q} \betaa_{r,p}^{\ell}))^2},
                                                                                                                                                                                         \end{array}
                                                                                                                                                                                           \end{equation} 
\noindent in practice $f(.)=\exp(.)$ and   $\frac{\partial J}{\partial \w_{q,p}^{\ell}}$ is obtained from layerwise gradient backpropagation (as already integrated in standard deep learning tools including PyTorch and TensorFlow).  Hence,  $\frac{\partial J}{\partial \betaa_{q,p}^{\ell}}$ is obtained by multiplying the original gradient $\frac{\partial J}{\partial {\w}_{q,p}^{\ell}}$ by $\frac{\exp(\sum_{r} \betaa_{r,p}^{\ell})}{(\exp( \betaa_{q,p}^{\ell})+\exp(\sum_{r\neq q} \betaa_{r,p}^{\ell}))^2}$.
\subsection{Pooling}
If pooling  on regular grids (or vectorial data in general) is well defined, it is not the case for graphs  \cite{spectral_coarsening}. As a consequence, most of GCN architectures do not include pooling layers in their architectures \cite{kipf17, matrix_completion} excepting a few attempts which try to incorporate pooling in a non explicit way using multi-level graph coarsening  (i.e., by reducing graphs by a factor of two at each level and describing each node by the average or the max of its descendants \cite{Bresson16,graculus} or by using clustering  \cite{SpasePooling,DiffPooling} and reordering  \cite{SetOr,DGCNN,ECC,PointNet}).  For highly irregular graphs (e.g., with heterogeneous degrees), this graph coarsening process usually results into imbalanced hierarchical representations and this substantially affects the accuracy of the learned  graph representations.  In practice, existing methods  (for instance \cite{Bresson16}) add fake nodes in the input graphs in order to rebalance the coarsening process. However,  fake nodes are spurious  and this may lead to contaminated graph representations after coarsening.  Besides, this pooling process is not invariant to node permutations and  node reordering (based on automorphisms) cannot guarantee permutation invariance  for general and irregular graphs.\\
\indent In this work, we consider an alternative solution in order to achieve pooling. Our method relies on two steps:  an expansion-step is first achieved at the node level followed by a global average pooling in order to achieve permutation invariance. Note that the first step (expansion) is necessary in order to generate high dimensional (and sparse) node representations and hence preserve the discrimination power of nodes before applying the second step of global average pooling. Put differently, without expansion, average pooling achieves permutation invariance but dilutes node information and this results into less discriminant graph representations as shown in experiments (see also \cite{sahbicvpr08a,sahbipami11,sahbiijmir15,sahbiicvs13,sahbicbmi08}).    \\
\indent Considering $\N_r(v)$ as the set of $r$-hop neighbors of a given node $v \in \V$ and $\N_r(v)=\cup_{l=1}^L \N_r^l(v)$ as the union of $L$ subsets\footnote{In practice, each subset $\N_r^l(v)$ includes only nodes with labels equal to $l$ (see again node labels in Fig.~\ref{fig1}).}, the expansion of $v$ is defined as
 \begin{equation}
\phi(v) \leftarrow \left((\psi \star g_\theta)(v), \ \ \   \frac{1}{|\N_r^1(v)|}  \sum_{v' \in \N_r^1(v)} (\psi \star g_\theta)(v'), \dots,  \frac{1}{|\N_r^L(v)|}  \sum_{v' \in \N_r^L(v)} (\psi \star g_\theta)(v')\right).
\end{equation} 
\noindent For a large and fine-grained neighborhood system $\N_r(v)=\cup_{l=1}^L \N_r^l(v)$  (i.e., $r\geq 1$ and $L\gg1$), the expansion $\phi(v)$ takes into account not only the immediate neighbors of $v$ but also a large extent and this results into high dimensional, sparse and  discriminating representations. Finally, a global average pooling is performed (as  $\sum_{v \in \V} \phi(v)$) to achieve permutation invariance prior to the softmax fully connected classification layer (see again Fig.~\ref{fig2}).

\section{Experiments}
We evaluate the performance of our multi-laplacian graph convolutional networks (MLGCN) on the challenging task of action recognition, using two standard datasets: SBU kinect \cite{SBU12} and UCF-101  \cite{UCF101}. SBU is an interaction dataset acquired (under relatively well controlled conditions) using the Microsoft Kinect sensor; it includes in total 282 video sequences belonging to 8 categories:  ``approaching'', ``departing'', ``pushing'', ``kicking'', ``punching'', ``exchanging objects'', ``hugging'', and ``hand shaking''. In contrast, UCF-101 is larger and more challenging; it includes 13,320 video shots belonging to 101 categories with variable duration, poor frame resolution, viewpoint and illumination changes, occlusion, cluttered background and eclectic content ranging  from multiple and highly interacting individuals to single and completely passive ones. In all these experiments, we use the same evaluation protocols as the ones suggested in \cite{UCF101,SBU12} (i.e., split2 for UCF-101 and train-test split for SBU) and we report the average accuracy over all the classes of actions.

\begin{table}[ht]
\begin{center}
\resizebox{0.95\textwidth}{!}{
 \begin{tabular}{cc|c|ccccccccccccc|c}
  &  &  Binary & \multicolumn{13}{c}{Binary $\times$ Gaussian} &\multicolumn{1}{|c}{Multi-lap} \\
  &  &            & $10^{-6}\sigma$ & $10^{-5}\sigma$ &  $10^{-4}\sigma$ & $10^{-3}\sigma$ & $10^{-2}\sigma$ & $10^{-1}\sigma$  & $\sigma$ & $10 \sigma$ &  $10^2\sigma$ & $10^3\sigma$ & $10^4\sigma$ & $10^5\sigma$ &  $10^6\sigma$ & \\
  \hline
  \hline
   \multirow{3}{*}{\rotatebox{35}{Unormalized}} & $k=1$ & 93.00 & 92.32 &92.32 & 92.32 &92.32& 92.32 & 92.32 & 92.32  & 92.32 & 92.30 & 92.30 & 92.30  & 92.30 & 92.30 & 93.41 \\
                              & $k=4$  & 89.25  & 88.87 & 88.87& 88.87  & 88.87 & 88.87 &88.87 & 88.87  & 88.87 & 88.87 & 88.86 &88.86 &88.86 & 88.86 & 90.07 \\
                              & $k=32$  & 86.00 & 86.31 & 86.31& 86.31  & 86.31 & 84.31 &86.31 & 86.31  & 86.32 & 86.32 & 86.32 &86.32 &86.32 & 86.32 & 86.91 \\
  \hline
  \multirow{3}{*}{\rotatebox{35}{Normalized}} & $k=1$  & 93.00 & 92.28 & 92.28 & 92.28 &92.28 & 92.28 &92.28 &92.26  & 92.26 & 92.26 & 92.26 & 92.26  & 92.28 & 92.28 & 93.49 \\
                                   & $k=4$  &90.00 &89.36 & 89.36& 89.36  & 89.36 & 89.36 &89.36 & 89.36  & 89.38 & 89.38 & 89.39 &89.37 &89.37 & 89.37 & 91.49 \\
                                   & $k=32$  & 88.00 & 88.31 & 88.31& 88.31  & 88.31 & 88.31 &88.31 & 88.31  & 88.32 & 88.32 & 88.32 &88.32 &88.32 & 88.32 & 89.21 \\
 \hline
   \multirow{3}{*}{\rotatebox{35}{Random w}}  & $k=1$ &93.00 & 92.05& 92.05 & 92.06  & 92.05 &92.05 & 92.05& 92.05  &92.09 & 92.09 & 92.09 & 92.06  & 92.06 & 92.06 & 93.46  \\
                                   & $k=4$  &96.00 & 94.06 & 94.06 & 94.06  & 94.00 &94.00 & 94.00 & 94.01  & 94.00 &94.01 &94.00 & 94.00  & 94.00 &94.00 & 96.31 \\
                                   & $k=32$  &96.00 & 94.03 & 94.03& 94.03 & 94.03 &94.03 & 94.03 & 94.03  & 94.03 &94.02 &94.02 & 94.02  &94.02&94.02 & 96.29 \\
  \hline
      Multi-lap              &         &97.15& 94.61 &94.58 & 94.61   & 94.63& 94.63& 94.63 & 94.62  &94.63 &94.63 &94.63 &94.63  &94.63 &94.63 & \bf{98.6}\\
                         
 \end{tabular}}
\end{center}\caption{Performances on SBU for different elementary laplacians (normalized, unormalized and random walk) and their marginal and total combinations using MLGCN (note that our expansion+GP is used for pooling).  In this table, "binary"  means that $\A^k$ is used to build the elementary laplacian while "binary $\times$ gaussian" means that ``$\A^k \times$ gaussian similarity'' is used instead; for each graph $\G$, the scale  $\sigma$  of the gaussian similarity is taken as the average distance between node features in $\G$. See also table 5 in supplementary material \cite{SMBM19} including  results {\it without expansion}.}

\label{table1}
\end{table}
\begin{table}[ht]
\begin{center}
\resizebox{0.95\textwidth}{!}{
\begin{tabular}{cc|c|ccccccccccccc|c}
  &  &  Binary & \multicolumn{13}{c}{Binary $\times$ Gaussian} &\multicolumn{1}{|c}{Multi-lap} \\
  &  &            & $10^{-6}\sigma$ & $10^{-5}\sigma$ &  $10^{-4}\sigma$ & $10^{-3}\sigma$ & $10^{-2}\sigma$ & $10^{-1}\sigma$  & $\sigma$ & $10 \sigma$ &  $10^2\sigma$ & $10^3\sigma$ & $10^4\sigma$ & $10^5\sigma$ &  $10^6\sigma$ & \\
  \hline
  \hline
  \multirow{3}{*}{\rotatebox{35}{Unormalized}} & $k=1$ & 55.32 & 50.67 &  50.67 &  50.67  &  50.68 &  50.70 & 50.70 & 50.70  &  50.71 &  50.72 &  50.72 &  50.72  & 50.70 &  50.70 & 56.55  \\
                              & $k=4$  &  59.23 &  55.22 &  55.22 &  55.22  &  55.22 &  55.20 & 55.20 & 55.20  &  54.95 &  54.96 &  54.95 &  54.98  &  55.00 &  54.98 &  60.05 \\
                              & $k=32$  & 55.10 &  52.05 &  52.05 &  52.05  &  52.05 &  52.11 &  52.11 &  52.11  &  52.11 &  52.11 &  52.06 &  52.06  &  52.06 &  52.08 &  56.48\\
  \hline
  \multirow{3}{*}{\rotatebox{35}{Normalized}} & $k=1$  & 55.6 & 50.78 &  50.77 & 50.27  &  50.42 &  50.40 &  50.42 &  50.42  &  50.42 &  50.42 & 50.42 &  50.42  &  50.42 &  50.42 & 56.80 \\
                                   & $k=4$  & 59.45 & 55.32 &  55.35 &  55.35  &  55.00 & 55.00 & 54.60 &  54.60  &  54.60 &  54.60 & 54.60 & 54.60  &  54.60 &  54.60 & 60.35 \\
                                   & $k=32$  & 55.25 & 51.19 &  51.19 &  51.19  & 49.78 &  49.79 & 49.79 & 49.79  & 49.79 & 49.78 & 49.78 &  49.77  & 49.77 & 49.77 &  56.52
                            \\
 \hline
   \multirow{3}{*}{\rotatebox{35}{Random w}}      & $k=1$ & 60.09 &  58.00 &  58.00 &  57.98  & 58.00 & 58.00& 58.00 &  58.00  & 58.01 &  57.95 &  57.95 & 57.95  & 57.92 & 57.94 &  60.85 \\
                                   & $k=4$  &  61.63 & 58.05 &  58.05 & 58.05  &  58.05 &  58.05&  58.02 &  58.02  &  58.02 &  57.98 &  57.98 &  57.98  &  57.98 &  58.02 &  61.90 \\
                                   & $k=32$  &  60.23 &  58.02 & 58.02 &  58.02  &  58.02 &  58.01& 58.02 & 58.02  &  58.01 &  57.95 &  57.95 &  57.95  & 57.92 &  57.92 & 60.9 \\
  \hline
      Multi-lap              &         & 62.00  &  58.24 &  58.16 &  58.14  & 58.14 & 58.14 &  58.15 & 58.15  &  58.13&  58.14& 58.16 &  58.18  & 58.15 &  58.17 & \bf{63.27} \\
                         
\end{tabular}}
\end{center}
\caption{Performance on UCF;  see caption of  table \ref{table1} for the settings.  See also table 6 in supplementary material \cite{SMBM19} including  results {\it without expansion}.}\label{table2}

\end{table}

\indent  We trained our MLGCN for 150 epochs on UCF-101 (and  40 on SBU) using the PyTorch SGD optimizer and we  set the learning rate to 0.0006 (decayed by a factor  0.1 after 100 epochs) for UCF-101 and 0.7 for SBU. We set the batch size to $30$ and the Chebyshev order $K$ to $4$ using grid search and cross validation. All these experiments are run on  GPUs;  Tesla P100 (with 16 Go) for UCF-101 and  Titan X Pascal (with 12 Go) for SBU. No data augmentation is achieved.  Tables~\ref{table1} and \ref{table2} show a comparison of action recognition performances, using MLGCN against different baselines involving individual laplacians (normalized, unormalized, random walk built on top of different affinity matrices and scale parameters). In these tables, we show the results using expansion and global average pooling (GP). We also show in table~\ref{tab3}(a-c) the results for (i) different pooling strategies (no-pooling, only GP, feature propagation~\cite{SGCCONV19} and feature propagation+GP), (ii) various multi-laplacian depths and activation functions\footnote{As shown in table \ref{tab3}(b), performances improve/stabilize very quickly, as the depth increases, since  the size of the training set is limited compared to the large number of training parameters in the MLP of the multi-laplacian. These performances are consistently better when using leaky ReLU (compared  to ReLU) and this is explained by the modeling capacity of the former. Indeed, leaky ReLU reflects better the (positive and negative) values of our laplacians while ReLU cuts off all the negative values.} and (iii) different input graph descriptions (for SBU).   From all these results, we observe a clear and a consistent gain of MLGCN w.r.t all the individual laplacian settings; this gain is further amplified when using ``expansion+GP'' with a large spatial extent and a fine-grained neighborhood system $\N_r(v)=\cup_{l=1}^L \N_r^l(v)$  (i.e., $r\geq 1$ and $L\gg1$). This gain results from the {\it complementary aspects of the used elementary laplacians} and also the  {\it match} between the topological properties of the learned multiple laplacians and the actual topology of the manifolds enclosing the input graphs. Besides,  ``expansion+GP'' aggregates the representations of the learned GCN filters in a way that maintains their high discrimination power (at the node level) while achieving permutation invariance. The latter is clearly necessary especially when handing videos  with multiple interacting persons that frequently appear in interchangeable orders  (as in SBU and UCF).\\
\indent  Finally, we compare the classification performances   of our MLGCN against related methods ranging from standard machine learning ones (SVMs  \cite{TP19,SBU12}, sequence based such as LSTM and GRU \cite{DeepGRU,GCALSTM,STALSTM}, 2D/3D CNNs \cite{pose,kin3d,twostream14,TP19}  including  appearance and motion streams) to deep graph (no-vectorial) methods based on spatial and spectral convolution \cite{kipf17,SGCCONV19,Bresson16}. From the results in table \ref{compare},  MLGCN brings a substantial gain w.r.t state of the art graph-based methods on both sets, and provides comparable results with the best vectorial methods on SBU.  On UCF, while  vectorial methods are highly effective, their combination with our MLGCN (through a late fusion) brings an extra gain despite the fact that bridging the -- last few percentage -- gap is challenging, and this clearly shows its complementary aspect. 

\begin{table}[!htb]
\centering
\resizebox{0.38\linewidth}{!}{
    \begin{tabular}{c||cc|cc}
    
      Pooling  & \multicolumn{2}{c}{Single-lapl} & \multicolumn{2}{|c}{Multi-lap} \\
               &  SBU &  UCF  & SBU & UCF \\       
      \hline
      \hline 
            No pooling& 93.94 & 59.16 & 95.70& 61.20  \\
      Global Pooling (GP) & 93.90&59.10 & 95.62&61.17  \\               
         Features prop \cite{SGCCONV19} & 94.27&59.30 & 96.36&61.31  \\
         Features prop \cite{SGCCONV19} + GP & 94.30&59.26 &96.43&61.25  \\   \hline
         Exp ($r=1$, $L=1$)+GP  & 94.15&59.20 & 96.35&61.25  \\
         Exp ($r=2$, $L=1$)+GP& 94.32&59.33 & 96.42&61.30  \\ 
         Exp ($r=1$, $L=n$)+GP& \bf{96.00}&\bf{60.54 }& {\bf98.60}& \bf 63.27  \\        
    \end{tabular}} \\
    {\scriptsize (a)} \vspace*{0.5cm}

\resizebox{0.29\linewidth}{!}{
    \begin{tabular}{c||cc|cc}
Depth&   \multicolumn{2}{c}{Leaky ReLU}   &   \multicolumn{2}{|c}{ReLU} \\  
      & SBU & UCF & SBU & UCF \\
      \hline
      \hline
   1    & \textbf{98.60}&63.10 & 98.57&63.07     \\
   2    & 98.56&\textbf{63.27} & 98.52&63.25     \\  
   3     & 98.30&63.27 & 98.23&63.23  \\
   \end{tabular}} \\
 {\scriptsize (b)}                                 \vspace*{0.5cm}

               \resizebox{0.33\linewidth}{!}{
    \begin{tabular}{c||c}
        skeleton representation & accuracy\\
      \hline
      \hline
      Cloud of joints&  31.65/34.25 \\
         Spatio-temporel skeletons&  36.10/38.00  \\
         Orthocentred joints&  43.25/45.80   \\
         Cylindrical features \cite{cylindrical2,Beyer87} & 38.42/40.10      \\
          3D  coord $+$velocity features \cite{velocity1}   & 38.50/40.20 \\
        Joint joint orientation  \cite{GeoFeat17}   & 74.95/ 76.20\\
        Joint line distance \cite{GeoFeat17}   &85.60/ 87.50\\
           \hline
        Our temporal chunking (sec 2)    & \bf{96.00}/ \bf{98.60}  
 
              \end{tabular}} \\
  {\scriptsize (c) } \vspace*{0.5cm}
\\
  \caption{(a) Behavior of our MLGCN with and without expansion, i.e., after its ablation and replacement with other pooling methods. Note that results with the best single laplacians taken from tables \ref{table1} and \ref{table2} are also shown. (b) Behavior of our MLGCN w.r.t different depths and activation functions. (c) Performance of MLGCN on SBU for different state of the art skeleton graph/node representations; again results are also shown for the best underlying single laplacians (taken from tables \ref{table1} and \ref{table2}).  In this table, "Cloud of joints" stand for graphs based on the similarity between all the keypoints of different frames; "Spatio-temporel skeleton" graphs are obtained by computing intra-frame joint similarity and by connecting them to their predecessors and successors through frames;  "Orthocentered joints" are obtained by centering the keypoint coordinates of each skeleton in each frame. Details about the other used node features (namely "Cylindrical features", "3D  coord $+$ velocity features", " Joint joint orientation" and   "Joint line distance") can be found in  \cite{cylindrical1,cylindrical2,Beyer87,velocity1,velocity2,GeoFeat17}.  }\label{tab3}
\end{table}

\begin{table}[!htb]
\resizebox{1.02\linewidth}{!}{
  \begin{tabular}{ccccc|cccccccccccccc|ccccc|cccccccccccccc}
  \hline
    \multicolumn{19}{c}{SBU}  & \multicolumn{12}{|c}{UCF} \\
   \hline 
     \multicolumn{5}{c}{Graph methods}  & \multicolumn{14}{|c}{Vectorial methods} &  \multicolumn{5}{|c}{Graph methods}  & \multicolumn{7}{|c}{Vectorial methods}  \\    
       &  &  &  &  &  &  &  &  &  &       &  &  &  &  &  &  &  &  &  &        &  &  &  &  &  &  &  &  &  &    \\
    \rotatebox{90}{90.00} &  \rotatebox{90}{96.00} &  \rotatebox{90}{94.00}&  \rotatebox{90}{96.00}&  \rotatebox{90}{\bf{98.60} }&   \rotatebox{90}{49.7 }&  \rotatebox{90}{80.3 }&  \rotatebox{90}{86.9 }&  \rotatebox{90}{83.9 }&  \rotatebox{90}{80.35 }&  \rotatebox{90}{90.41}&   \rotatebox{90}{93.3 }&  \rotatebox{90}{\textbf{99.02} }&  \rotatebox{90}{90.5}&   \rotatebox{90}{91.51}&  \rotatebox{90}{94.9}&  \rotatebox{90}{97.2}&  \rotatebox{90}{95.7}&  \rotatebox{90}{93.7 }&  \rotatebox{90}{59.39 }&  \rotatebox{90}{61.11}&  \rotatebox{90}{62.81}&  \rotatebox{90}{60.54}&  \rotatebox{90}{\bf{63.27} }&  \rotatebox{90}{68.58  }&  \rotatebox{90}{64.38 }&  \rotatebox{90}{70.37 }&  \rotatebox{90}{68.05 }&  \rotatebox{90}{77.34 }&  \rotatebox{90}{79.10 }&    \rotatebox{90}{91.12}&  \rotatebox{90}{93.20}&  \rotatebox{90}{95.60}& \rotatebox{90}{95.92} & \rotatebox{90}{96.41}&  \rotatebox{90}{96.60} & \rotatebox{90}{\bf{97.94}}  \\
   &  &  &  &  &  &  &  &  &  &       &  &  &  &  &  &  &  &  &  &        &  &  &  &  &  &  &  &  &  &    \\
   \rotatebox{90}{ GCNConv \cite{kipf17}} & \rotatebox{90}{ArmaConv \cite{ARMACONV19}} & \rotatebox{90}{ SGCConv \cite{SGCCONV19}} & \rotatebox{90}{ ChebyNet \cite{Bresson16}} & \rotatebox{90}{Our best MLGCN setting}  & \rotatebox{90}{  Raw coordinates  \cite{SBU12}} & \rotatebox{90}{Joint features \cite{SBU12}} & \rotatebox{90}{Interact Pose \cite{InteractPose}} & \rotatebox{90}{CHARM \cite{CHARM15}} & \rotatebox{90}{ HBRNN-L \cite{HBRNNL15}} & \rotatebox{90}{Co-occurence LSTM \cite{CoOccurence16}} & \rotatebox{90}{ ST-LSTM \cite{STLSTM16}} & \rotatebox{90}{ Joint line distance\cite{GeoFeat17}} & \rotatebox{90}{ Topological pose ordering\cite{velocity2}} & \rotatebox{90}{ STA-LSTM \cite{STALSTM}} & \rotatebox{90}{ GCA-LSTM \cite{GCALSTM}} & \rotatebox{90}{ VA-LSTM  \cite{VALSTM}} & \rotatebox{90}{DeepGRU  \cite{DeepGRU}} & \rotatebox{90}{ Riemannian manifold traj \cite{RiemannianManifoldTraject}} &  \rotatebox{90}{ GCNConv \cite{kipf17}} & \rotatebox{90}{ArmaConv \cite{ARMACONV19}} & \rotatebox{90}{ SGCConv \cite{SGCCONV19}} & \rotatebox{90}{ ChebyNet \cite{Bresson16}} & \rotatebox{90}{Our best MLGCN setting} & \rotatebox{90}{ Temporal pyramid \cite{TP19}} & \rotatebox{90}{Potion \cite{pose}}  &
    \rotatebox{90}{Temporal pyramid  \cite{TP19} + our best MLGCN setting}  &
    \rotatebox{90}{Potion \cite{pose} + our best MLGCN setting}  &
    \rotatebox{90}{Temporal pyramid  \cite{TP19} + Potion \cite{pose}}  & \rotatebox{90}{Temporal pyramid  \cite{TP19} + Potion \cite{pose}+ our best MLGCN}  &
    \rotatebox{90}{  2D two stream \cite{twostream14}}  &\rotatebox{90}{  2D two stream \cite{twostream14}+Our best MLGCN setting}  
    &\rotatebox{90}{ 3D appearance \cite{kin3d}} &
    \rotatebox{90}{ 3D appearance \cite{kin3d}+Our best MLGCN setting } & \rotatebox{90}{ 3D motion \cite{kin3d}} &\rotatebox{90}{ 3D motion \cite{kin3d}+ our best MLGCN setting} &\rotatebox{90}{    3D two stream \cite{kin3d}} &  
 \end{tabular}}\caption{Comparison against state of the art methods.}   \label{compare}            
\end{table}
\section{Conclusion} 
We introduced in this paper a novel Multi-Laplacian Graph Convolutional Network (MLGCN) for action recognition. The strength of our method resides in its effectiveness in learning combined laplacian convolutional operators each one dedicated to a particular setting of the manifold enclosing the input graph data. Our method also considers a novel pooling process which first expands nodes with their context prior to achieve global average pooling. Extensive experiments conducted on the SBU as well as the challenging UCF-101 datasets, show the outperformance and also the complementary aspect of our MLGCN w.r.t different baselines and the related work including graph-based methods.\\
As a future work, we are currently studying other laplacian combination strategies and also the extension of our graph convolutional networks to other tasks and benchmarks.


\begin{thebibliography}{10}
\providecommand{\url}[1]{#1}
\csname url@samestyle\endcsname
\providecommand{\newblock}{\relax}
\providecommand{\bibinfo}[2]{#2}
\providecommand{\BIBentrySTDinterwordspacing}{\spaceskip=0pt\relax}
\providecommand{\BIBentryALTinterwordstretchfactor}{4}
\providecommand{\BIBentryALTinterwordspacing}{\spaceskip=\fontdimen2\font plus
\BIBentryALTinterwordstretchfactor\fontdimen3\font minus
  \fontdimen4\font\relax}
\providecommand{\BIBforeignlanguage}[2]{{%
\expandafter\ifx\csname l@#1\endcsname\relax
\typeout{** WARNING: IEEEtran.bst: No hyphenation pattern has been}%
\typeout{** loaded for the language `#1'. Using the pattern for}%
\typeout{** the default language instead.}%
\else
\language=\csname l@#1\endcsname
\fi
#2}}
\providecommand{\BIBdecl}{\relax}
\BIBdecl

\bibitem{du2015hierarchical}
Y.~Du, W.~Wang, and L.~Wang, ``Hierarchical recurrent neural network for
  skeleton based action recognition,'' in \emph{Proceedings of the IEEE
  conference on computer vision and pattern recognition}, 2015, pp. 1110--1118.

\bibitem{sahbi2011context}
H.~Sahbi, J.-Y. Audibert, and R.~Keriven, ``Context-dependent kernels for
  object classification,'' \emph{IEEE transactions on pattern analysis and
  machine intelligence}, vol.~33, no.~4, pp. 699--708, 2011.

\bibitem{ling2015}
L.~Wang, ``kernel machines for video action recognition,'' in \emph{PhD thesis,
  Telecom ParisTech, Paris-Saclay University}, 2015.

\bibitem{postadjian2017investigating}
T.~Postadjian, A.~Le~Bris, H.~Sahbi, and C.~Mallet, ``Investigating the
  potential of deep neural networks for large-scale classification of very high
  resolution satellite images,'' \emph{ISPRS Annals}, vol.~4, pp. 183--190,
  2017.

\bibitem{temporalpyramid}
H.~Pirsiavash and D.~Ramanan, ``Detecting activities of daily living in
  first-person camera views,'' in \emph{Computer Vision and Pattern Recognition
  (CVPR), 2012 IEEE Conference on}.\hskip 1em plus 0.5em minus 0.4em\relax
  IEEE, 2012, pp. 2847--2854.

\bibitem{sahbi2002face}
H.~Sahbi, D.~Geman, and N.~Boujemaa, ``Face detection using coarse-to-fine
  support vector classifiers,'' in \emph{Image Processing. 2002. Proceedings.
  2002 International Conference on}, vol.~3.\hskip 1em plus 0.5em minus
  0.4em\relax IEEE, 2002, pp. 925--928.

\bibitem{segment_net}
L.~Wang, Y.~Xiong, Z.~Wang, Y.~Qiao, D.~Lin, X.~Tang, and L.~V. Gool,
  ``Temporal segment networks: Towards good practices for deep action
  recognition,'' \emph{ECCV}, 2016.

\bibitem{chen2006human}
H.-S. Chen, H.-T. Chen, Y.-W. Chen, and S.-Y. Lee, ``Human action recognition
  using star skeleton,'' in \emph{Proceedings of the 4th ACM international
  workshop on Video surveillance and sensor networks}.\hskip 1em plus 0.5em
  minus 0.4em\relax ACM, 2006, pp. 171--178.

\bibitem{wang2013directed}
L.~Wang and H.~Sahbi, ``Directed acyclic graph kernels for action
  recognition,'' in \emph{Proceedings of the IEEE International Conference on
  Computer Vision}, 2013, pp. 3168--3175.

\bibitem{temporal_pyramid}
D.~Xu and S.-F. Chang, ``Visual event recognition in news video using kernel
  methods with multi-level temporal alignment,'' \emph{CVPR}, 2007.

\bibitem{sahbi2013cnrs}
H.~Sahbi, ``Cnrs-telecom paristech at imageclef 2013 scalable concept image
  annotation task: Winning annotations with context dependent svms.'' in
  \emph{CLEF (Working Notes)}, 2013.

\bibitem{wang2015action}
L.~Wang, Y.~Qiao, and X.~Tang, ``Action recognition with trajectory-pooled
  deep-convolutional descriptors,'' in \emph{Proceedings of the IEEE conference
  on computer vision and pattern recognition}, 2015, pp. 4305--4314.

\bibitem{gowayyed2013histogram}
M.~A. Gowayyed, M.~Torki, M.~E. Hussein, and M.~El-Saban, ``Histogram of
  oriented displacements (hod): Describing trajectories of human joints for
  action recognition,'' in \emph{Twenty-Third International Joint Conference on
  Artificial Intelligence}, 2013.

\bibitem{sahbi2003coarse}
H.~Sahbi, ``Coarse-to-fine support vector machines for hierarchical face
  detection,'' Ph.D. dissertation, PhD thesis, Versailles University, 2003.

\bibitem{superived_dic_action}
H.~Wang, C.~Yuan, W.~Hu, and C.~Sun, ``Supervised class-specific dictionary
  learning for sparse modeling in action recognition,'' \emph{Pattern
  Recognition}, vol.~45, no.~11, pp. 3902--3911, 2012.

\bibitem{multi_svm}
C.~Schuldt, I.~Laptev, and B.~Caputo, ``Recognizing human actions: a local svm
  approach,'' in \emph{Pattern Recognition, 2004. ICPR 2004. Proceedings of the
  17th International Conference on}, vol.~3.\hskip 1em plus 0.5em minus
  0.4em\relax IEEE, 2004, pp. 32--36.

\bibitem{boujemaa2001ikona}
N.~Boujemaa, J.~Fauqueur, M.~Ferecatu, F.~Fleuret, V.~Gouet, B.~Saux, and
  H.~Sahbi, ``Ikona: Interactive generic and specific image retrieval,'' in
  \emph{Proceedings of the International workshop on Multimedia Content-Based
  Indexing and Retrieval (MMCBIR?2001)}, 2001, pp. 25--29.

\bibitem{wang2012}
J.~Wang, Z.~Liu, Y.~Wu, and J.~Yuan, ``Mining actionlet ensemble for action
  recognition with depth cameras,'' \emph{In Computer Vision and Pattern
  Recognition (CVPR), 2012 IEEE Conference on (pp. 1290-1297)}, 2012.

\bibitem{sahbi2007kernel}
H.~Sahbi, ``Kernel pca for similarity invariant shape recognition,''
  \emph{Neurocomputing}, vol.~70, no. 16-18, pp. 3034--3045, 2007.

\bibitem{fathi2008}
A.~Fathi and G.~Mori, ``Action recognition by learning mid-level motion
  features.'' \emph{In Computer Vision and Pattern Recognition, 2008. CVPR
  2008. IEEE Conference on (pp. 1-8). IEEE.}, 2008.

\bibitem{napoleon20102d}
T.~Napol{\'e}on and H.~Sahbi, ``From 2d silhouettes to 3d object retrieval:
  contributions and benchmarking,'' \emph{Journal on Image and Video
  Processing}, vol. 2010, p.~1, 2010.

\bibitem{wang2016action}
P.~Wang, Z.~Li, Y.~Hou, and W.~Li, ``Action recognition based on joint
  trajectory maps using convolutional neural networks,'' in \emph{Proceedings
  of the 24th ACM international conference on Multimedia}.\hskip 1em plus 0.5em
  minus 0.4em\relax ACM, 2016, pp. 102--106.

\bibitem{xia2012}
L.~Xia, C.~Chen, and J.~K. Aggarwal, ``View invariant human action recognition
  using histograms of 3d joints.'' \emph{In Computer vision and pattern
  recognition workshops (CVPRW), 2012 IEEE computer society conference on (pp.
  20-27). IEEE.}, 2012.

\bibitem{ferecatu2008telecomparistech}
M.~Ferecatu and H.~Sahbi, ``Telecomparistech at imageclefphoto 2008: Bi-modal
  text and image retrieval with diversity enhancement.'' in \emph{CLEF (Working
  Notes)}, 2008.

\bibitem{ali2010}
S.~Ali and M.~Shah, ``Human action recognition in videos using kinematic
  features and multiple instance learning.'' \emph{IEEE transactions on pattern
  analysis and machine intelligence, 32(2), 288-303.}, 2010.

\bibitem{murthy2013ordered}
O.~Murthy and R.~Goecke, ``Ordered trajectories for large scale human action
  recognition,'' in \emph{Proceedings of the IEEE international conference on
  computer vision workshops}, 2013, pp. 412--419.

\bibitem{boujemaa2004visual}
N.~Boujemaa, F.~Fleuret, V.~Gouet, and H.~Sahbi, ``Visual content extraction
  for automatic semantic annotation of video news,'' in \emph{the proceedings
  of the SPIE Conference, San Jose, CA}, vol.~6, 2004.

\bibitem{Le2011}
Q.~V. Le, W.~Zou, S.~Y. Yeung, and A.~Y. Ng, ``Learning hierarchical invariant
  spatio-temporal features for action recognition with independent subspace
  analysis.'' \emph{In Computer Vision and Pattern Recognition (CVPR), 2011
  IEEE Conference on (pp. 3361-3368). IEEE.}, 2011.

\bibitem{tollari2008comparative}
S.~Tollari, P.~Mulhem, M.~Ferecatu, H.~Glotin, M.~Detyniecki, P.~Gallinari,
  H.~Sahbi, and Z.-Q. Zhao, ``A comparative study of diversity methods for
  hybrid text and image retrieval approaches,'' in \emph{Workshop of the
  Cross-Language Evaluation Forum for European Languages}.\hskip 1em plus 0.5em
  minus 0.4em\relax Springer, 2008, pp. 585--592.

\bibitem{Matikainen2009}
P.~Matikainen, M.~Hebert, and R.~Sukthankar, ``Trajectons: Action recognition
  through the motion analysis of tracked features.'' \emph{In Computer Vision
  Workshops (ICCV Workshops), 2009 IEEE 12th International Conference on (pp.
  514-521). IEEE.}, 2009.

\bibitem{li2011superpixel}
X.~Li and H.~Sahbi, ``Superpixel-based object class segmentation using
  conditional random fields,'' in \emph{Acoustics, Speech and Signal Processing
  (ICASSP), 2011 IEEE International Conference on}.\hskip 1em plus 0.5em minus
  0.4em\relax IEEE, 2011, pp. 1101--1104.

\bibitem{yu2010}
T.~Yu, T.~Kim, and R.~Cipolla, ``Real-time action recognition by spatiotemporal
  semantic and structural forests.'' \emph{In BMVC (Vol. 2, No. 5, p. 6).},
  2010.

\bibitem{mkl_action}
L.~Chen, L.~Duan, and D.~Xu, ``Event recognition in videos by learning from
  heterogeneous web sources,'' in \emph{Proceedings of the IEEE Conference on
  Computer Vision and Pattern Recognition}, 2013, pp. 2666--2673.

\bibitem{poppe2010}
R.~Poppe, ``A survey on vision-based human action recognition.'' \emph{Image
  and vision computing, 28(6), 976-990.}, 2010.

\bibitem{sahbi2015imageclef}
H.~Sahbi, ``Imageclef annotation with explicit context-aware kernel maps,''
  \emph{International Journal of Multimedia Information Retrieval}, vol.~4,
  no.~2, pp. 113--128, 2015.

\bibitem{Iosifidis2013}
A.~Iosifidis, A.~Tefas, and I.~Pitas, ``Minimum class variance extreme learning
  machine for human action recognition.'' \emph{IEEE Transactions on Circuits
  and Systems for Video Technology, 23(11), 1968-1979.}, 2013.

\bibitem{wu2012view}
X.~Wu and Y.~Jia, ``View-invariant action recognition using latent kernelized
  structural svm,'' in \emph{European conference on computer vision}.\hskip 1em
  plus 0.5em minus 0.4em\relax Springer, 2012, pp. 411--424.

\bibitem{sahbi2002coarse}
H.~Sahbi and N.~Boujemaa, ``Coarse-to-fine support vector classifiers for face
  detection,'' in \emph{null}.\hskip 1em plus 0.5em minus 0.4em\relax IEEE,
  2002, p. 30359.

\bibitem{wu2011action}
S.~Wu, O.~Oreifej, and M.~Shah, ``Action recognition in videos acquired by a
  moving camera using motion decomposition of lagrangian particle
  trajectories,'' in \emph{2011 International conference on computer
  vision}.\hskip 1em plus 0.5em minus 0.4em\relax IEEE, 2011, pp. 1419--1426.

\bibitem{pyramid_kernel}
K.~Grauman and T.~Darrell, ``The pyramid match kernel: Efficient learning with
  sets of features,'' \emph{JMLR}, 2007.

\bibitem{danafar2007action}
S.~Danafar and N.~Gheissari, ``Action recognition for surveillance applications
  using optic flow and svm,'' in \emph{Asian Conference on Computer
  Vision}.\hskip 1em plus 0.5em minus 0.4em\relax Springer, 2007, pp. 457--466.

\bibitem{wang2014bags}
L.~Wang and H.~Sahbi, ``Bags-of-daglets for action recognition,'' in
  \emph{Image Processing (ICIP), 2014 IEEE International Conference on}.\hskip
  1em plus 0.5em minus 0.4em\relax IEEE, 2014, pp. 1550--1554.

\bibitem{kin3d}
J.~Carreira and A.~Zisserman, ``Quo vadis, action recognition? a new model and
  the kinetics dataset,'' in \emph{Computer Vision and Pattern Recognition
  (CVPR), 2017 IEEE Conference on}.\hskip 1em plus 0.5em minus 0.4em\relax
  IEEE, 2017, pp. 4724--4733.

\bibitem{pose}
V.~Choutas, P.~Weinzaepfel, J.~Revaud, and C.~Schmid, ``Potion: Pose motion
  representation for action recognition,'' \emph{CVPR}, 2018.

\bibitem{jiu2017nonlinear}
M.~Jiu and H.~Sahbi, ``Nonlinear deep kernel learning for image annotation,''
  \emph{IEEE Transactions on Image Processing}, vol.~26, no.~4, pp. 1820--1832,
  2017.

\bibitem{baccouche2011}
M.~Baccouche, F.~Mamalet, C.~Wolf, C.~Garcia, and A.~Baskurt, ``Sequential deep
  learning for human action recognition.'' \emph{In International Workshop on
  Human Behavior Understanding (pp. 29-39). Springer, Berlin, Heidelberg},
  2011.

\bibitem{spresnet16}
C.~Feichtenhofer, A.~Pinz, and R.~Wildes, ``Spatiotemporal residual networks
  for video action recognition,'' in \emph{Advances in neural information
  processing systems}, 2016, pp. 3468--3476.

\bibitem{spresnetmulti17}
C.~Feichtenhofer, A.~Pinz, and R.~P. Wildes, ``Spatiotemporal multiplier
  networks for video action recognition,'' in \emph{2017 IEEE Conference on
  Computer Vision and Pattern Recognition (CVPR)}.\hskip 1em plus 0.5em minus
  0.4em\relax IEEE, 2017, pp. 7445--7454.

\bibitem{thiemert2006using}
S.~Thiemert, H.~Sahbi, and M.~Steinebach, ``Using entropy for image and video
  authentication watermarks,'' in \emph{Security, Steganography, and
  Watermarking of Multimedia Contents VIII}, vol. 6072.\hskip 1em plus 0.5em
  minus 0.4em\relax International Society for Optics and Photonics, 2006, p.
  607218.

\bibitem{bourdis2011constrained}
N.~Bourdis, M.~Denis, and H.~Sahbi, ``Constrained optical flow for aerial image
  change detection,'' in \emph{2011 IEEE International Geoscience and Remote
  Sensing Symposium (IGARSS)}, 2011, pp. 4176--4179.

\bibitem{thiemert2005applying}
S.~Thiemert, H.~Sahbi, and M.~Steinebach, ``Applying interest operators in
  semi-fragile video watermarking,'' in \emph{Security, Steganography, and
  Watermarking of Multimedia Contents VII}, vol. 5681.\hskip 1em plus 0.5em
  minus 0.4em\relax International Society for Optics and Photonics, 2005, pp.
  353--363.

\bibitem{sahbi2010context}
H.~Sahbi and X.~Li, ``Context-based support vector machines for interconnected
  image annotation,'' in \emph{Asian Conference on Computer Vision}.\hskip 1em
  plus 0.5em minus 0.4em\relax Springer, 2010, pp. 214--227.

\bibitem{vo2014}
P.~Vo, ``Transductive inference and kernel methods for image annotation,'' in
  \emph{PhD thesis, Telecom ParisTech, Paris-Saclay University}, 2014.

\bibitem{sahbi2008context}
H.~Sahbi, J.-Y. Audibert, J.~Rabarisoa, and R.~Keriven, ``Context-dependent
  kernel design for object matching and recognition,'' in \emph{CVPR}, 2008,
  pp. 1--8.

\bibitem{oliveau2018}
Q.~Oliveau, ``Learning models for object and ship category recognition,'' in
  \emph{PhD thesis, Telecom ParisTech, Paris-Saclay University}, 2018.

\bibitem{jiu2015semi}
M.~Jiu and H.~Sahbi, ``Semi supervised deep kernel design for image
  annotation,'' in \emph{Acoustics, Speech and Signal Processing (ICASSP), 2015
  IEEE International Conference on}.\hskip 1em plus 0.5em minus 0.4em\relax
  IEEE, 2015, pp. 1156--1160.

\bibitem{imagenet}
J.~Deng, W.~Dong, R.~Socher, L.-J. Li, K.~Li, and L.~Fei-Fei, ``Imagenet: A
  large-scale hierarchical image database,'' \emph{IEEE Computer Vision and
  Pattern Recognition (CVPR)}, 2009.

\bibitem{art_imagenet}
B.~Zoph, V.~Vasudevan, J.~Shlens, and Q.-V. Le, ``Learning transferable
  architectures for scalable image recognition,'' \emph{CVPR}, 2018.

\bibitem{MKL}
M.~Gönen and E.~Alpaydın, ``Multiple kernel learning algorithms,''
  \emph{JMLR}, 2011.

\bibitem{ucf}
K.~Soomro, A.-R. Zamir, and M.~Shah, ``Ucf101: A dataset of 101 human action
  classes from videos in the wild,'' \emph{CRCV-TR-12-01}, 2012.

\bibitem{MKL_alignement}
C.cortes, M.~Mohri, and A.~Rostamizadeh, ``Algorithms for learning kernels
  based on centered alignement,'' \emph{J.Mach.Learn, Rev.}, vol.~13, pp.
  795--828, 2012.

\end{thebibliography}


\newpage
\begin{thebibliography}{9}
   \bibitem{Bresson16}
 M. Defferrard, X. Bresson, P. Vandergheynst. Convolutional Neural Networks on graphs with Fast Localized Spectral Filtering. In Neural Information Processing
Systems (NIPS), 2016
  
  \bibitem{kipf17} 
TN. Kipf, M. Welling. Semi-supervised classification with graph convolutional networks. In International Conference on Learning Representations (ICLR), 2017

\bibitem{sahbiaccv2010}
  H. Sahbi, X. Li. Context-based support vector machines for interconnected image annotation. Asian Conference on Computer Vision. Springer, Berlin, Heidelberg, 2010.
  

\bibitem{Henaff15}
M. Henaff, J. Bruna, Y. LeCun. Deep Convolutional Networks on Graph-Structured Data. In arXiv:1506.05163, 2015


\bibitem{tristan2017}
T. Postadjian, A. Le Bris, H. Sahbi, C. Mallet. Investigating the potential of deep neural networks for large-scale classification of very high resolution satellite images. ISPRS Annals 4, 183-190, 2017.

\bibitem{bruna13}
  J. Bruna, W. Zaremba, A. Szlam, Y. Lecun. Spectral Networks and Deep Locally Connected Networks on Graphs. In International Conference on Learning Representations (ICLR), 2014

\bibitem{Bronstein17}
M. M. Bronstein, J. Bruna, Y. LeCun, A. Szlam, P. Vandergheynst. Geometric deep learning: going beyond Euclidean data. In IEEE Signal Processing Magazine, 2017

\bibitem{Bronstein16}
D. Boscaini, J. Masci, E. Rodolà, M. M. Bronstein, Learning shape correspondence with anisotropic convolutional neural networks. In Neural Information Processing
Systems (NIPS), 2016 

\bibitem{sahbipr2012}
F. Yuan, G-S. Xia, H. Sahbi, V. Prinet.  Mid-level Features and Spatio-Temporal Context for Activity Recognition. Pattern Recognition.  volume 45, number 12,  4182-4191, 2012
 
 \bibitem{sahbikpca06}
H Sahbi.  Kernel PCA for similarity invariant shape recognition. Neurocomputing 70 (16-18), 3034-3045

\bibitem{Monti17}
F. Monti, D. Boscaini, J. Masci, E. Rodolà, J. Svoboda, and M. Bronstein. Geometric deep learning on graphs and manifolds using mixture
model cnns. In computer Vision and Pattern Recognition (CVPR), 2017

\bibitem{Petar18}
P. Velickovic, G. Cucurull, A. Casanova, A. Romero, P. Li,and Y. Bengio. Graph attention networks. In International Conference on Learning Representations (ICLR), 2018

\bibitem{Ortega13}
 D-I. Shuman, S-K. Narang, P. Frossard, A. Ortega, P. Vandergheynst. The emerging field of signal processing on graphs: Extending high-dimensional data analysis to networks and other irregular domains. In IEEE Signal Processing Magazine, 2013

\bibitem{Ortega18}
A. Ortega, P. Frossard, J. Kovacevic, Jose M. F. Moura, P. Vandergheynst. Graph Signal Processing: Overview, Challenges, and Applications. In Proceedings of the IEEE ( Volume: 106 , Issue: 5 , May 2018 ) Page(s): 808 - 828.

\bibitem{sahbicassp11}
  X. Li, H. Sahbi. Superpixel-based object class segmentation using conditional random fields. IEEE International Conference on Acoustics, Speech and Signal Processing (ICASSP). 2011.

  
 \bibitem{Yaser17}
  Z. Cao, T. Simon, S-E. Wei, Y. Sheikh. Realtime Multi-Person 2D Pose Estimation using Part Affinity Fields. In computer Vision and Pattern Recognition (CVPR), 2017 
  
  \bibitem{pose}
  V. Choutas, P. Weinzaepfel, J. Revaud, C. Schmid. PoTion: Pose MoTion Representation for Action Recognition. In computer Vision and Pattern Recognition (CVPR), 2018
  

 \bibitem{twostream14}
 K. Simonyan, A. Zisserman. Two-Stream Convolutional Networks for Action Recognition in Videos. In Neural Information Processing
Systems (NIPS), 2014
 
  \bibitem{sahbispie2004}
N Boujemaa, F Fleuret, V Gouet, H Sahbi. Visual content extraction for automatic semantic annotation of video news. The proceedings of the SPIE Conference, San Jose, CA 6, 2004. 


   \bibitem{kin3d}
 J. Carreira, A. Zisserman. Quo Vadis, Action Recognition? A New Model and the Kinetics Dataset. In Computer Vision and Pattern Recognition (CVPR), 2017
   
   \bibitem{temporalpyramid}
 H. Pirsiavash, D. Ramanan. Detecting Activities of Daily Living in First-person Camera Views. In Computer Vision and Pattern Recognition (CVPR), 2012

\bibitem{mkl_action}
 L. Chen, L. Duan, and D. Xu. Event Recognition in Videos by Learning From Heterogeneous Web Sources. In Computer Vision and Pattern Recognition (CVPR), 2013
  
  
\bibitem{sahbiigarss11}
N. Bourdis, D. Marraud, H. Sahbi. Constrained optical flow for aerial image change detection. IEEE International Geoscience and Remote Sensing Symposium, 4176-4179, 2011. 
 


  \bibitem{temporal_pyramid}
 D. Xu, S-F. Chang. Visual Event Recognition in News Video using Kernel Methods with Multi-Level Temporal Alignment. In Computer Vision and Pattern Recognition (CVPR), 2013
  2007
  
 
\bibitem{superived_dic_action}
 H. Wang, C. Yuan, W. Hu, and C. Sun. Supervised class-specific dictionary learning for sparse modeling in action recognition. Pattern Recognition , 2012


\bibitem{sahbiijmir15}
H Sahbi. Imageclef annotation with explicit context-aware kernel maps. International Journal of Multimedia Information Retrieval 4 (2), 113-128


\bibitem{multi_svm}
 C. Schuldt, I. Laptev, B. Caputo. Recognizing human actions: a local SVM approach. In International Conference on Pattern Recognition Systems (ICPR), 2004
   
  \bibitem{spresnet16}
 C. Feichtenhofer, A. Pinz, R-P. Wildes. Spatiotemporal Residual Networks for Video Action Recognition. In Neural Information Processing Systems (NIPS), 2016
 
 \bibitem{spresnetmulti17}
 C. Feichtenhofer, A. Pinz, R-P. Wildes. Spatiotemporal Multiplier Networks for Video Action Recognition. In Computer Vision and Pattern Recognition (CVPR), 2017 
 
 
 

\bibitem{sahbiclef13}
  H. Sahbi. CNRS-TELECOM ParisTech at ImageCLEF 2013 Scalable Concept Image Annotation Task: Winning Annotations with Context Dependent SVMs. CLEF (Working Notes). 2013.


\bibitem{STGCN}
S. Yan and Y. Xiong and D. Lin. Spatial Temporal Graph Convolutional Networks for Skeleton-Based Action Recognition.  In Association for the Advancement of Artificial Intelligence (AAAI), 2018

\bibitem{motion_graph}
M. Schlichtkrull, T-N. Kipf, P. Bloem, R. Van den Berg, I. Titov, M. Welling. Modeling relational data with graph convolutional networks. In International Conference on Machine Learning (ICML), 2019

\bibitem{matrix_completion}
R. Van den Berg, T-N. Kipf, M. Welling. Graph Convolutional Matrix Completion. In arXiv:1706.02263, 2017
 

 \bibitem{graculus}
 I. Dhillon, Y. Guan, and B. Kulis.  Weighted Graph Cuts Without Eigenvectors: A Multilevel Approach. In IEEE Transactions on Pattern Analysis and Machine Intelligence (PAMI), 29(11):1944–1957, 2007

\bibitem{sahbiigarss16}
Q. Oliveau, H. Sahbi. Attribute learning for ship category recognition in remote sensing imagery. IEEE International Geoscience and Remote Sensing Symposium (IGARSS), 96-99, 2016.

\bibitem{mklimage2017}
M. Jiu, H. Sahbi. Nonlinear deep kernel learning for image annotation. IEEE Transactions on Image Processing, volume 26, number 4, 1820-1832, 2017.

 \bibitem{SAGECONV18}
 WL. Hamilton, R. Ying, J. Leskovec. Inductive Representation Learning on Large Graphs. In Neural Information Processing Systems (NIPS), 2018
 
 \bibitem{GRAPHCONV18}
 C. Morris, M. Ritzert, M. Fey, WL. Hamilton, JE. Lenssen, G. Rattan, M. Grohe. Weisfeiler and Leman Go Neural: Higher-order Graph Neural Networks. In arXiv:1810.02244, 2018
 
 \bibitem{GAT18}
 
P. Velickovic, G. Cucurull, A. Casanova, A. Romero, P. Lio, Y. Bengio. Graph Attention Networks. In International Conference on Learning Representation (ICLR), 2018
  
\bibitem{AGNNCONV18}
 K. Thekumparampil, C.  Wang, S. Oh, L-J. Li. Attention-based Graph Neural Network for Semi-supervised Learning. In arXiv:1803.03735, 2018


\bibitem{sahbiicpr18}
M. Jiu, H. Sahbi, L. Qi. Deep Context Networks for Image Annotation,  24th International Conference on Pattern Recognition (ICPR), 2422-2427, 2018. 


\bibitem{ARMACONV19}
F-M. Bianchi, D. Grattarola, C. Alippi, L. Livi. Graph Neural Networks with Convolutional ARMA Filters. In arXiv:1901.01343, 2019

\bibitem{SGCCONV19}
F. Wu, T. Zhang, A. Holanda de Souza Jr., C. Fifty, T. Yu, K-Q. Weinberger. Simplifying Graph Convolutional Networks. In arXiv:1902.07153, 2019

\bibitem{APPNP19}
J. Klicpera, A. Bojchevski, S. Günnemann. Predict then Propagate: Graph Neural Networks meet Personalized PageRank. In International Conference on Learning Representation (ICLR), 2019


\bibitem{velocity1}
M. Zanfir, M. Leordeanu, and C. Sminchisescu. The moving pose: An efficient 3d kinematics descriptor for low-latency action recognition and detection. In International Conference on Computer Vision (ICCV), 2013

 \bibitem{sahbicbmi08}
H. Sahbi, JY. Audibert, J. Rabarisoa, R. Keriven. Object recognition and retrieval by context dependent similarity kernels. International Workshop on Content-Based Multimedia Indexing, 216-223, 2008.  

\bibitem{velocity2}
F. Baradel, C. Wolf, J. Mille. Pose-conditioned Spatio-Temporal Attention for Human Action Recognition. In arXiv preprint, 2017

\bibitem{cylindrical1}
D. Weinland, R. Ronfard, E. Boyer. Free viewpoint action recognition using motion history volumes. In Computer vision and image understanding 104 (2-3), 249-257

\bibitem{cylindrical2}
Q. Ke, M. Bennamoun, S. An, F. Sohel, F. Boussaid. A New Representation of Skeleton Sequences for 3D Action Recognition. In Computer Vision and Pattern Recognition (CVPR), 2017

\bibitem{Beyer87}
Beyer, W. H. CRC Standard Mathematical Tables, 28th ed. Boca Raton, FL: CRC Press, 1987

\bibitem{GeoFeat17}
S. Zhang, X. Liu, J. Xiao. On Geometric Features for Skeleton-Based Action Recognition using Multilayer LSTM Networks. In Conference on Applications of Computer Vision (WACV), 2017

\bibitem{sahbiphd}
H. Sahbi.  Coarse-to-fine support vector machines for hierarchical face detection. PhD thesis, Versailles University, 2003. 

\bibitem{InteractPose}
Y. Ji, G. Ye, and H. Cheng. Interactive body part contrast mining for human interaction recognition. In International Conference on Multimedia and Expo Workshops (ICMEW), 2014

\bibitem{CHARM15}
W. Li, L. Wen, M. Choo Chuah, and S. Lyu. Category-blind human action recognition: A practical recognition system. In International Conference on Computer Vision, 2015

\bibitem{sahbiicip18}
H. Sahbi. Structured Scene Decoding with Finite State Machines. 25th IEEE International Conference on Image Processing (ICIP), 485-489, 2018. 


\bibitem{HBRNNL15}
Y. Du, W. Wang, and L. Wang. Hierarchical recurrent neural network for skeleton based action recognition. In Computer Vision and Pattern Recognition (CVPR), 2015

\bibitem{sahbiicip09}
M. Ferecatu, H. Sahbi. Multi-view object matching and tracking using canonical correlation analysis. 16th IEEE International Conference on Image Processing (ICIP), 2109-2112, 2009. 

\bibitem{CoOccurence16}
W. Zhu, C. Lan, J. Xing, W. Zeng, Y. Li, L. Shen, and X. Xie. Co-occurrence feature learning for skeleton based action recognition using regularized deep LSTM networks. In Association for the Advancement of Artificial Intelligence (AAAI), 2016

\bibitem{STLSTM16}
J. Liu, A. Shahroudy, D. Xu, and G. Wang. Spatio-temporal LSTM with trust gates for 3D human action recognition. In European Conference on Computer Vision (ECCV), 2016

\bibitem{STALSTM}
S. Song, C. Lan, J. Xing, W. Zeng, and J. Liu. An end-to end spatio-temporal attention model for human action recognition from skeleton data. In Association for the Advancement of Artificial Intelligence (AAAI), 2017

\bibitem{sahbipami11}
H Sahbi, JY Audibert, R Keriven. Context-dependent kernels for object classification. IEEE transactions on pattern analysis and machine intelligence 33 (4), 699-708


\bibitem{GCALSTM}
J. Liu, G. Wang, L. Duan, K. Abdiyeva, and A. C.
Kot. Skeleton-based human action recognition with global context-aware attention lstm networks. IEEE Transactions on Image Processing, 27(4):1586–1599, April 2018

\bibitem{VALSTM}
P. Zhang, C. Lan, J. Xing, W. Zeng, J. Xue, and N. Zheng. View adaptive recurrent neural networks for high performance human action recognition from skeleton data. In International Conference on Computer Vision (ICCV), 2017

\bibitem{DeepGRU}
M. Maghoumi, JJ. LaViola Jr. DeepGRU: Deep Gesture Recognition Utility. In arXiv preprint arXiv:1810.12514, 2018


\bibitem{lingsahbieccv2014}
 L. Wang, H. Sahbi. Nonlinear Cross-View Sample Enrichment for Action Recognition. European Conference on Computer Vision. Springer, 2014.


\bibitem{RiemannianManifoldTraject}
A. Kacem, M. Daoudi, B. Ben Amor, S. Berretti, J-Carlos.  Alvarez-Paiva. A Novel Geometric Framework on Gram Matrix Trajectories for Human Behavior Understanding.  IEEE Transactions on Pattern Analysis and Machine Intelligence, 28 September 2018


\bibitem{SplineCNN}
M. Fey, J-E. Lenssen, F. Weichert, H. Müller. SplineCNN: Fast geometric deep learning with continuous B-spline kernels. In Computer Vision and Pattern Recognition (CVPR), 2018


\bibitem{lingsahbi2013}
 L. Wang, H. Sahbi. Directed Acyclic Graph Kernels for Action Recognition. Proceedings of the IEEE International Conference on Computer Vision. 2013.
\bibitem{ASGCN}
L. Shi, Y. Zhang, J. Cheng, H. Lu. Adaptive Spectral Graph Convolutional Networks for Skeleton-Based Action Recognition. In arXiv preprint arXiv:1805.07694, 2018

\bibitem{TSN}
L. Wang, Y. Xiong, Z. Wang, Y. Qiao, D. Lin, X. Tang, and L. Van Gool. Temporal segment networks: Towards good practices for deep action recognition. In European Conference on Computer Vision (ECCV), 2016

\bibitem{PartGraph}
K. Thakkar, P-J. Narayanan. Part-based Graph Convolutional Network for Action Recognition. In British Machine Vision Conference (BMVC), 2018

\bibitem{TP19}
A. Mazari, H. Sahbi. Deep Temporal Pyramid Design for Action Recognition. In international Conference on Acoustics, Speech and Signal Processing (ICASSP), 2019

\bibitem{Node_sampling}
Z. Wu and al. A Comprehensive Survey on Graph Neural Networks. In arXiv:1901.00596, 2019


\bibitem{wavelet_GCN}
B. Xu, H. Shen, Q. Cao, Y. Qiu, X. Cheng. Graph Wavelet Neural Network. In International Conference of Learning Representation (ICLR), 2019

\bibitem{sahbijstars17}
Q. Oliveau, H. Sahbi. Learning attribute representations for remote sensing ship category classification.  IEEE Journal of Selected Topics in Applied Earth Observations and Remote Sensing, 2017. 




\bibitem{scattering_GCN}
D. Zou, G. Lerman. Graph Convolutional Neural Networks via Scattering. In arXiv preprint arXiv:1804.00099, 2018

\bibitem{diff_scattering}
F. Gama, A. Ribeiro, J. Bruna. Diffusion Scattering Transforms on Graphs. In International Conference of Learning Representation (ICLR), 2019


\bibitem{sahbispie05}
S. Thiemert, H. Sahbi, M. Steinebach. Applying interest operators in semi-fragile video watermarking. Security, Steganography, and Watermarking of Multimedia Contents VII. Vol. 5681. International Society for Optics and Photonics, 2005.

\bibitem{DGCNN}
M. Zhang, Z. Cui, M. Neumann, Y. Chen. An End-to-End Deep Learning Architecture for Graph Classification. In Association for the Advancement of Artificial Intelligence (AAAI), 2018


\bibitem{SetOr}
O. Vinyals, S. Bengio, M. Kudlur. Order Matters: Sequence to sequence for sets. In International Conference on Learning Representation (ICLR), 2016

\bibitem{SpasePooling}
C. Cangea, P. Velickovic, N. Jovanovic, T. Kipf, P. Lio. Towards sparse hierarchical graph classifiers. In arXiv preprint arXiv:1811.01287, 2018

\bibitem{sahbipr19}
M. Jiu, H. Sahbi. Deep representation design from deep kernel networks. Pattern Recognition 88, 447-457, 2019. 


\bibitem{ECC}
M. Simonovsky, N. Komodakis. Dynamic edge-conditioned filters in convolutional neural networks on graphs. In Computer Vision and Pattern Recognition (CVPR), 2018

\bibitem{PointNet}
CR. Qi, L. Yi, H. Su, LJ Guibas. Pointnet++: Deep hierarchical feature learning on point sets in a metric space. In Neural Information Processing Systems (NIPS), 2017

\bibitem{sahbiicassp15}
M. Jiu, H. Sahbi. Semi supervised deep kernel design for image annotation. IEEE International Conference on Acoustics, Speech and Signal Processing, 2015 . 


\bibitem{DiffPooling}
R. Ying, J. You, C. Morris, X. Ren, W-L. Hamilton, J. Leskovec. Hierarchical Graph Representation Learning with Differentiable Pooling. In Neural Information Processing
Systems (NIPS), 2018

\bibitem{resnet_imagenet}
K. He, X. Zhang, S. Ren, and J. Sun. Deep residual learning for image recognition. In Computer Vision and Pattern Recognition (CVPR), 2016.


\bibitem{sahbiicassp16a}
M. Jiu, H. Sahbi. Deep kernel map networks for image annotation.   IEEE International Conference on Acoustics, Speech and Signal Processing, 2016. 

\bibitem{monte_carlo}
P. Hermosilla, T. Ritschel, PP. Vazquez, A. Vinacua, T. Ropinski. Monte Carlo convolution for learning on non-uniformly sampled point clouds. In ACM SIGGRAPH, 2018 

\bibitem{irregularity_GSP}
B. Girault, A. Ortega, S. Narayanan. Irregularity-aware graph Fourier transforms. IEEE Transactions on Signal Processing 66 (21), 5746-5761

\bibitem{vertex_analysis}
DI. Shuman, B. Ricaud, P. Vandergheynst. Vertex-frequency analysis on graphs. Applied and Computational Harmonic Analysis 40 (2), 260-291, 2016

\bibitem{sahbiicassp16b}
M. Jiu, H. Sahbi. Laplacian deep kernel learning for image annotation. IEEE International Conference on Acoustics, Speech and Signal Processing, 2016. 



\bibitem{covariant_GNN}
R. Kondor, HT. Son, H. Pan, B. Anderson, S. Trivedi. Covariant compositional networks for learning graphs. In arXiv preprint arXiv:1801.02144, 2017

\bibitem{Laplacian-topology}
X. Dong, D. Thanou, P. Frossard, P. Vandergheynst. Learning Laplacian matrix in smooth graph signal representations. IEEE Transactions on Signal Processing 64 (23), 6160-6173, 2016

\bibitem{spectral_coarsening}
A. Loukas, P. Vandergheynst. Spectrally approximating large graphs with smaller graphs. In Internationa Conference on Machine Learning (ICML), 2018


\bibitem{sahbiiccv17}
H. Sahbi. Coarse-to-fine deep kernel networks. Proceedings of the IEEE International Conference on Computer Vision, 1131-1139, 2017. 



\bibitem{lingsahbiicip2014}
L. Wang, H. Sahbi. Bags-of-Daglets for Action Recognition. IEEE International Conference on Image Processing (ICIP), 2014.

\bibitem{softplus1}
  X. Glorot, A. Bordes and Y. Bengio. Deep sparse rectifier neural networks. In International Conference on 
Artificial Intelligence and Statistics (AISTATS), 2011

\bibitem{softplus2}
  C. Dugas, Y. Bengio, F. Belisle, C. Nadeau, R. Garcia. Incorporating second-order functional knowledge for better option pricing. In Neural Information Processing
Systems (NIPS), 2000

\bibitem{leaky} 
H. Kaiming, Z. Xiangyu, R. Shaoqing, S. Jian. Delving Deep into Rectifiers: Surpassing Human-Level Performance on Image Net Classification. arXiv:1502.01852, 2015.


\bibitem{sahbiclef08}
S. Tollari, P. Mulhem, M. Ferecatu, H. Glotin, M. Detyniecki, P. Gallinari, H. Sahbi,  Z-Q. Zhao. A comparative study of diversity methods for hybrid text and image retrieval approaches. In Workshop of the Cross-Language Evaluation Forum for European Languages, pp. 585-592. Springer, Berlin, Heidelberg, 2008.

\bibitem{sahbiicvs13}
H. Sahbi. Explicit context-aware kernel map learning for image annotation. International Conference on Computer Vision Systems, 304-313, 2013




\bibitem{berg84}
Berg, Christian, Jens Peter Reus Christensen, and Paul Ressel. Harmonic analysis on semigroups: theory of positive definite and related functions. Vol. 100. New York: Springer, 1984

\bibitem{Shoenberg38}
Schoenberg, Isaac J. Metric spaces and positive definite functions. Transactions of the American Mathematical Society 44.3 (1938): 522-536

\bibitem{SMBM19}
A. Mazari, H. Sahbi. BMVC19  Supplementary Material. http://www-ia.lip6.fr/$\sim$sahbi/SMBM19.pdf

\bibitem{SBU12}
K. Yun and J. Honorio and D. Chattopadhyay and T-L. Berg and D. Samaras. Two-person Interaction Detection Using Body-Pose Features and Multiple Instance Learning. In Computer Vision and Pattern Recognition (CVPR), 2012

\bibitem{UCF101}
K. Soomro, A-R. Zamir and M. Shah. UCF101: A Dataset of 101 Human Action Classes From Videos in The Wild. In Computer Vision and Pattern Recognition (CVPR), 2012.

\bibitem{sahbicvpr08a}
H. Sahbi, J-Y. Audibert, J. Rabarisoa, R. Keriven, R. Context-dependent kernel design for object matching and recognition. In IEEE Conference on Computer Vision and Pattern Recognition (CVPR), 2008.



\end{thebibliography}
\end{document}